\documentclass[twoside]{article}
\usepackage{authblk}
\usepackage{amsfonts, amsthm}
\usepackage{cancel}
\usepackage{graphicx}
\usepackage{xcolor}         
\usepackage{hyperref}

\usepackage{blindtext}
\usepackage[accepted]{aistats2022}
\usepackage{lipsum}

\newtheorem{theorem}{Theorem}[section]

\newtheorem{lemma}[theorem]{Lemma}
\newtheorem{remark}[theorem]{Remark}
\newtheorem{proposition}[theorem]{Proposition}

\begin{document}

%
\runningtitle{On the Implicit Bias of Gradient Descent for Temporal Extrapolation}

%

\twocolumn[

\aistatstitle{On the Implicit Bias of Gradient Descent \\ for Temporal Extrapolation}

\aistatsauthor{ Edo Cohen-Karlik$^*$ \And Avichai Ben David$^*$ \And  Nadav Cohen \And Amir Globerson }

\aistatsaddress{ \And \vspace{-2mm} \\ Blavatnik School of Computer Science \vspace{1mm} \\ Tel Aviv University \And  }]


\begin{abstract}
When using recurrent neural networks (RNNs) it is common practice to apply trained models to sequences longer than those seen in training. This ``extrapolating'' usage deviates from the traditional statistical learning setup where guarantees are provided under the assumption that train and test distributions are identical.
 Here we set out to understand when RNNs can extrapolate, focusing on a simple case where the data generating distribution is memoryless. We first show that even with infinite training data, there exist RNN models that interpolate perfectly (i.e., they fit the training data) yet extrapolate poorly to longer sequences. We then show that if gradient descent is used for training, learning will converge to perfect extrapolation under certain assumptions on initialization.
Our results complement recent studies on the implicit bias of gradient descent, showing that it plays a key role in extrapolation when learning temporal prediction models.
 \end{abstract}


\section{INTRODUCTION}\label{sec:intro}

Practical deep neural networks are often larger than necessary for perfectly fitting the data they are trained on. This ``over-parametrized'' regime could potentially result in severe overfitting, but in practice neural networks tend to generalize surprisingly well to unseen data. This observation inspired a multitude of works aimed at theoretically understanding the phenomenon. It quickly became clear \cite{ZhangBHRV17} that the generalization arises from an ``implicit bias''~---~a tendency towards certain solutions that generalize well~---~induced by variants of gradient descent (GD) and their initialization schemes. Characterizing this implicit bias is a key goal in the theory of deep learning (e.g., see \cite{gunasekar2018characterizing} and \cite{WoodworthGLMSGS20}).

Most of the works studying implicit bias in deep learning consider the setting where train and test data are drawn from the same distribution, and focus on bounding the gap between train and test errors.  An equally interesting and complementary problem is that of ``extrapolation,'' which deals with how the learned function behaves outside the training distribution. A recent work has begun to explore this question in the context of fully connected neural networks \cite{xu2020neural}, showing that in the ``ultra wide'' regime, also known as Neural Tangent Kernel (NTK; \cite{jacot2020neural}) regime, learned functions extrapolate linearly. 

In this paper we focus on a setting where extrapolation is especially important: learning temporal models. In particular, we focus on recurrent neural networks (RNN), which form a standard tool for this task. It is common practice to train RNNs with sequences up to a certain length, and then apply them at inference time to longer sequences. The fact that this approach often works well in practice suggests that RNNs perform successful temporal extrapolation by virtue of the implicit bias of GD and its initialization scheme.  Our aim is to theoretically understand this phenomenon. 

Non-linear neural networks 
are notoriously difficult to analyze when operating outside the NTK regime, i.e.~for realistic model sizes. In order to make progress towards their theoretical understanding, researchers have turned to the simplified model of linear neural networks (e.g., see \cite{arora2018convergence,ji2019gradient}). Linear neural networks are trivial in terms of expressiveness (they realize only linear input-output mappings), but not so in terms of optimization and generalization: they induce highly non-convex training objectives, and exhibit phenomena akin to their non-linear counterparts. We accordingly base our analysis of temporal extrapolation on linear recurrent neural networks, also known in the literature as linear dynamical systems (LDS) \cite{lds}.

The setting we study is geared specifically towards understanding the temporal extrapolation implicitly brought forth by GD. We consider the case where training data is generated by a memoryless teacher, so as to avoid an explicit bias towards non-trivial extrapolation.
Learning from the training data via GD on an over-parametrized RNN, the extrapolation question boils down to whether a model trained on sequences of length~$k$ will realize a memoryless mapping for time steps greater than~$k$.
We show that in this setup, there exist RNN weights that perfectly fit the training data (regardless of how many training sequences were collected), and yet extrapolate poorly. Empirically however, we find that training via GD with standard initialization schemes yields solutions that extrapolate well. This clearly demonstrates an implicit bias towards good extrapolation. Interestingly, learned solutions do not comprise a zero state transition matrix, but rather one which is far from zero, yet still results in good extrapolation.

Any result establishing implicit bias towards specific solutions must entail assumptions on initialization (otherwise, one may initialize at any solution that perfectly fits the training data, and GD will remain there). In the context of linear neural networks, it is common to assume that initialization admits certain balancedness properties \cite{arora2018convergence}.
%
In the same spirit, our analysis focuses on an initialization where model weights satisfy certain symmetries. 
Under such initialization, we prove that if GD converges to a perfect fit of the training data, it will do so with an extrapolating (i.e., memoryless) solution. We show that this extrapolation is due to a certain form of ``complementary slackness'' between components of the model.
To the best of our knowledge, this result is the first to provide formal evidence for an implicit bias of GD towards temporal extrapolation outside the NTK regime.

The remainder of the paper is structured as follows. In Section \ref{sec:related_work} we discuss related work. Section \ref{sec:linear_rnns} describes the setup and required definitions. In Section \ref{sec:long_sequence_guarantee_extrapolation} we analyze the case of learning with sequences whose length~$k$ is larger than the hidden dimension of the model~$d$. Section \ref{sec:bad_solutions} shows that when $k< d$ there exist solutions that perfectly fit the training data but do not extrapolate to longer sequences. Section \ref{sec:implicit_bias_of_short_sequences} analyzes the solutions obtained by GD in the latter regime, showing they do extrapolate. Finally, section \ref{sec:experiments} provides experiments supporting our theoretical findings. 

\section{RELATED WORK}\label{sec:related_work}

Linear RNNs, also known as linear dynamical systems (LDS) have been studied for decades \cite{kalman1963mathematical, ghahramani1996parameter}. The aspect most related to this work is \textit{system identification} \cite{ljung1999system}, which studies the conditions under which the exact parameters of an LDS can be recovered. This is related to the question of temporal extrapolation because recovering the correct system parameters from training sequences of finite length will lead to perfect extrapolation. Note however, that works along this line do not analyze the dynamics of GD, but rather provide characterizations of the conditions under which system identification is possible.

Two related properties of an LDS that are necessary for unique identification are \textit{controllability} and \textit{observability} \cite{kalman1960general}. In the memoryless setting the learned LDS is neither controllable nor observable, and therefore its identification is an ill-posed problem not treated by classic approaches.
One common approach to identification are subspace methods \cite{ho1966effective}, which perform an SVD of the Hankel matrix (i.e., the matrix representing the linear input-output mapping realized by an LDS) in order to extract an approximation of the system which has a low dimensional state-space. Other approaches to low rank Hankel matrices are based on rank relaxations such as nuclear norm \cite{fazel2001rank,liu2010interior}. See also \cite{glover1984all} for additional approximation notions.

This paper focuses on the question of how GD learns a ``simple'' system from observed data. It may thus be viewed as the LDS analogue of works studying the implicit bias in matrix factorization \cite{arora2019implicit}. Such works ask whether GD over matrix factorization fits observations with solutions that minimize a complexity measure, and what that complexity measure is.
At present, these questions are largely open.

Our work relates to the recent results of \cite{hardt2016gradient} showing that GD can optimize the loss for LDS. 
Specifically, they showed that when training on sequences up to length~$k$, GD will converge to a dynamical system that approximates the impulse response up to time~$k$. However, this result does not imply extrapolation in the sense we consider here, as it does not guarantee approximation of the impulse response for times beyond~$k$.

Another recent work \cite{xu2020neural} studies extrapolation of deep learning beyond the support of the training data. They show that feed forward neural networks with ReLU activation extrapolate to linear functions outside the support, and further provide principles to construct graph neural networks which exhibit bias towards specific extrapolation. We study extrapolation in the temporal domain, providing insights into the solutions found by GD over linear RNNs. 

Another study of implicit bias in linear RNNs is \cite{emami2021implicit}, who examine the correlation between such models and convolutions in the asymptotic (NTK) regime of an infinite-dimensional hidden state.
They show that GD tends to short-term memory solutions. Our work differs from \cite{emami2021implicit} in several aspects. First, we consider the realistic case of finite-dimensional hidden states (i.e.~we operate outside the NTK regime). Second, we directly prove an extrapolation result, whereas 
\cite{emami2021implicit} provide results on structural biases of learned impulse responses.

\section{LINEAR RECURRENT NEURAL NETWORKS} \label{sec:linear_rnns}

In this section we describe our model and the analyzed setting. We consider a single layer linear RNN, and for simplicity present our analysis for the single-input single-output (SISO) case.  The analysis readily extends to the more general multiple-input multiple-output (MIMO) case, as shown in the Supplementary (Section~\ref{sec:appendix_mimo}).
The dynamics of interest are defined by:
\begin{equation}\label{eq:output_of_rnn}
    \hat{y}_t=C s_t,\qquad s_{t+1}=As_t+B x_{t+1},
\end{equation}
where $A\in\mathbb{R}^{d\times d}$, $B\in\mathbb{R}^{d\times 1}$ and $C\in\mathbb{R}^{1\times d}$ are learned parameters (weights), and $\{ s_t \}_{t=1}^\infty \subset \mathbb{R}^{1\times d}$ are the resulting hidden states, where by assumption $s_0=0$.
We refer to~$A$ as the \emph{state transition matrix}, to $B$ as the \emph{input weights}, and to~$C$ as the \emph{output weights}.

Given $\mathbf{x} = ( x_1 , x_2 , \ldots , x_k ) \in \mathbb{R}^k$, an input sequence of length~$k$, we denote by $RNN(\mathbf{x}) = \hat{y}_k$ the output of the RNN at time step~$k$.\footnote{We interchangeably use $RNN(\mathbf{x})=RNN(x_1,\dots,x_k)$.} The latter can be expressed in terms of the learned parameters and input sequence. 
Indeed, it results from taking a convolution of the input sequence with the impulse response sequence $(CA^i B)_{i=0}^\infty$:
\begin{equation}\label{eq:unroll_output}
    \hat{y}_k=\sum_{i=1}^{k} CA^{k-i}Bx_i.
\end{equation}

We consider the problem of learning the parameters of the RNN from a set of $N$ training sequences and their desired outputs: 
\begin{equation}\label{eq:training_data}
    S=\left\lbrace (\mathbf{x}^{(i)},y^{(i)}) \right\rbrace_{i=1}^N \subset \mathbb{R}^k \times \mathbb{R},
\end{equation}
via minimization of the empirical squared loss:
\begin{equation*}
\frac{1}{2N}\sum_{i=1}^N\left(RNN(\mathbf{x}^{(i)}) -y^{(i)} \right)^2.
\end{equation*}
Note that this setup is more challenging (entails less supervision) than the one in which training labels include outputs for all time steps between $1$ and~$k$. 
For convenience, in the remainder of the paper, we omit the subscript of the output at time~$k$, and simply denote $\hat{y}=\hat{y}_k$. 

Our interest lies on the impact of implicit bias on extrapolation.  To decouple that from generalization (i.e., from the question of how accurate the model is on sequences of length~$k$ not seen during training), we assume an unlimited amount of training data, or formally, that GD is applied to the population loss:
\begin{equation*}
    \mathbb{E}_{\mathbf{x},y}\left[\frac{1}{2}\left(RNN(\mathbf{x})-y\right)^2\right].
\end{equation*}
We will study the case where training data is generated by a teacher RNN, and to isolate the implicit bias of GD, we avoid any type of explicit bias towards non-trivial extrapolation.
That is, we assume the teacher RNN is \emph{\textbf{memoryless}}, meaning that there exists $w^* \in \mathbb{R}$ such that for any $k \in \mathbb{N}$ and any input sequence $\mathbf{x} = ( x_1 , \ldots , x_k ) \in \mathbb{R}^k$, the corresponding label is given by $y = w^* x_k$
(namely, the output depends on input only via latest time step). 
We disregard the trivial case of constant zero labels, i.e.~we assume $w^* \neq 0$.
Using Equation~\eqref{eq:unroll_output}, we obtain an expression for the loss induced by the memoryless teacher, in the case where inputs are drawn independently.
\begin{lemma}\label{lemma:expected_loss}
Assume $\mathbf{x}\sim \mathcal{D}$ such that $\mathbb{E}_{\mathcal{D}}[\mathbf{x}\mathbf{x}^\top]=I_k$, where $I_k \in \mathbb{R}^{k \times k}$ is the identity matrix. Then, given a memoryless teacher RNN, the loss for the student RNN satsifies:
\begin{align}\label{eq:expected_loss}
    \mathbb{E}_{\mathbf{x},y}\left[\frac{1}{2}\left(RNN(\mathbf{x})-y\right)^2\right]& =\\
    \frac{1}{2}\sum_{i=1}^{k-1}(CA^{k-i}B)^2 &+\frac{1}{2}(CB-w^*)^2\nonumber
\end{align}
\end{lemma}
\begin{proof}
The result follows from expanding the population loss and calculating first and second order moments of $\mathbf{x}$. See Supplementary.
\end{proof}

Lemma~\ref{lemma:expected_loss} admits a simple interpretation. It states that the population loss will be minimized when the impulse response starts with~$w^*$, and is followed by $k-1$ zeros.
This agrees with the fact that the system is trained to be memoryless for the first $k$ time steps. However, as we shall see later, it does not guarantee that it will be memoryless for times greater than $k$.

We say that the learned RNN \emph{\textbf{extrapolates}} with respect to the teacher RNN if it agrees with the latter's output for any input sequence of \emph{any length}, including lengths which exceed that of the training sequences.
This amounts to requiring that for any $j \in \mathbb{N}$ (in particular $j > k$) and any $\mathbf{x} = ( x_1 , \ldots , x_j ) \in \mathbb{R}^j$, the output of the learned RNN satisfies $\hat{y} = RNN ( \mathbf{x} ) = w^* x_j$.
%

\begin{lemma}\label{lemma:extrapolation_conditions}
The parameters $(A,B,C)$ for the learned RNN are extrapolating with respect to a memoryless teacher if and only if $CB=w^*$ and $CA^jB=0$ for all $j \in \mathbb{N}$.
\end{lemma}
\begin{proof}
The proof follows from Equation~\eqref{eq:unroll_output} and the expected loss in Lemma \ref{lemma:expected_loss} when $k\to\infty$.
\end{proof}

One possible solution that extrapolates to a memoryless teacher is $A=0$ along with any pair $B,C$ satsifying $CB=w^*$. Surprisingly, we observe empirically (see Section~\ref{sec:experiments}) that typically $A\neq0$ while $CA^jB=0$ for all $j \in \mathbb{N}$, suggesting a non-trivial alignment between $A$ and $B,C$.
In the next sections we theoretically explore this phenomenon.

\section{LONG TRAINING SEQUENCES GUARANTEE EXTRAPOLATION}\label{sec:long_sequence_guarantee_extrapolation}
Theorem~\ref{thm:cayley_hamilton} below shows that when $k>d$ (i.e., when the length of training sequences is larger than the width of the learned model), any solution that minimizes the loss extrapolates.  

\begin{theorem}\label{thm:cayley_hamilton}
Assume that $k > d$, and let $( A , B ,C )$ be a solution (parameters for learned RNN) that minimizes the loss in Equation~\eqref{eq:expected_loss}.
Then, it holds that $C A^j B = 0$ for all $j \in \mathbb{N}$, meaning the learned model extrapolates.
\end{theorem}
\begin{proof}
Let $p ( z ) = z^d + \rho_{d-1} z^{d-1} + \dots + \rho_1 z+\rho_0$, be the characteristic polynomial of the matrix $A \in \mathbb{R}^{d \times d}$.
By the Cayley-Hamilton theorem \cite{zhang1997quaternions, Frobenius1877}:
\[
p(A) = A^d + \rho_{d-1}A^{d-1} + \dots + \rho_1 A + \rho_0 I = 0,
\]
which implies that we may write:
\[
    A^d=-\sum_{i=0}^{d-1} \rho_i A^i.
\]
Multiplying both sides of the above by~$A$, followed by left multiplication by~$C$ and right multiplication by~$B$, yields:
\[
CA^{d+1}B=-\sum_{i=0}^{d-1} \rho_i CA^{i+1}B.
\]
Since the global minimum of the loss in Equation~\eqref{eq:expected_loss} is zero, it necessarily holds that $CA^jB=0$ for all $j\in\lbrace 1,\dots, k-1\rbrace$, and in particular for all $j\in\lbrace 1,\dots,d\rbrace$.
We therefore have:
\[
    CA^{d+1}B=-\sum_{i=0}^{d-1}c_i CA^{i+1}B=0.
\]
Continuing in this fashion, we conclude that $CA^jB=0$ for all $j \in \mathbb{N}$.
\end{proof}
The above result implies that sufficiently long training sequences guarantee extrapolation. In other words, the training data in this case is sufficient to uniquely identify the memoryless teacher. As we shall see next, for shorter training sequences this no longer holds.



\section{EXTRAPOLATION MAY FAIL FOR SHORT TRAINING SEQUENCES} \label{sec:bad_solutions}
In Section~\ref{sec:long_sequence_guarantee_extrapolation} we showed that when training sequences have length larger than the width of the trained model ($k > d$), learning guarantees extrapolation.
Proposition~\ref{prop:rnns_can_overfit} below shows that in stark contrast, when the training sequence length is no greater than model width ($k \leq d$), there exist solutions which minimize the training loss, and yet fail to extrapolate.
This implies that an arbitrary loss-minimizing learning algorithm may result in non-extrapolating solutions. 
In Section~\ref{sec:implicit_bias_of_short_sequences} we show that despite this fact, under certain conditions, solutions found by GD do extrapolate.

\begin{proposition}\label{prop:rnns_can_overfit}
For any training sequence length $k\ge 2$ and model width $d \ge k$, there exist RNN parameters $(A,B,C)$ that minimize the loss in Equation~\eqref{eq:expected_loss} but do not extrapolate.
\end{proposition}
\begin{proof}
Assume $d \ge k\ge 2$, and consider the following parameter setting for the learned RNN.
\[
 A=\begin{pmatrix}  0 & 0 & \dots & 0 & 0 & 1 \\
                    1 & 0 & \dots & 0 & 0 & 0\\
                    0 & 1 & \dots & 0 & 0 & 0\\
                    & & \ddots & & & \\
                    0 & 0 & \dots & 1 & 0 & 0 \\
                    0 & 0 & \dots & 0 & 1 & 0 
    \end{pmatrix} 
    \in \mathbb{R}^{d\times d} ~,
\]
\[
B=\begin{pmatrix}1\\ 0\\ \vdots\\ 0 \end{pmatrix} 
    \in \mathbb{R}^{d \times 1} ~,
\quad
C=(w^*,0,\dots,0) \in \mathbb{R}^{1 \times d}.
\]


Note that $A$ is a permutation matrix, and specifically, multiplying $A$ from the right by a general matrix $M\in\mathbb{R}^{d \times d}$ results in a cyclic shift of rows, i.e.:
\begin{equation*}
M=\begin{pmatrix} - & M_1 & -\\
 & \vdots & \\
- & M_d & -
\end{pmatrix}
,\quad
AM=\begin{pmatrix}
- & M_d & -\\
- & M_1 & -\\
  & \vdots & \\
- & M_{d-1} & -
\end{pmatrix}
\text{\,.}
\end{equation*}
Applying $A$ to itself, we have that for any $n\in\mathbb{N}$, $A^{nd}=I$ and consequently $CA^{nd}B=w^*\neq 0$, which contradicts extrapolation (by Lemma \ref{lemma:extrapolation_conditions}). On the other hand, for $j\in\lbrace 1,\dots,d-1\rbrace$ we have $CA^jB=0$ since the first row of $A^j$ is not $e_1$.\footnote{The first row of $A^j$ is given by $e_{d-j+1}$} To conclude, since $d\ge k$, the loss in Equation~\eqref{eq:expected_loss} is zero and therefore minimized, while the RNN does not meet the necessary condition for extrapolation.
\end{proof}



\section{IMPLICIT BIAS OF GRADIENT DESCENT} \label{sec:implicit_bias_of_short_sequences}
Section~\ref{sec:bad_solutions} showed that when the length of training sequences is no greater than the width of the trained model ($k \leq d$), there exist solutions which minimize the loss (achieve perfect generalization) and yet do not extrapolate.
Despite the existence of such non-extrapolating solutions, we observe empirically (see Section~\ref{sec:experiments}) that GD with standard initialization entails an implicit bias towards solutions that do extrapolate.
Theorem~\ref{thm:short_sequence_bias} below theoretically grounds this phenomenon, for the case where the input and output weights are initialized to the same value (i.e.~$B=C^\top$), and the state transition matrix~$A$ is initialized symmetrically.
This initialization captures the ``residual'' setting $A = I_d$, and more generally, allows~$A$ to have arbitrary magnitude, implying arbitrary distance from a trivial solution in which $A = 0$.
We emphasize that while our analysis assumes symmetric initialization, we observe empirical convergence to an extrapolating solution under non-symmetric initialization as well (see Section~\ref{sec:experiments}). 
Interestingly, we often see that the parameters of the model converge to a symmetric configuration even if not initialized this way (see Subsection~\ref{subsec:weight_dynamics}). Theoretically explaining this phenomenon is a promising direction for future work.

\begin{theorem}\label{thm:short_sequence_bias}
Assume $d \geq k > 2$, $w^*>0$, and that the learned RNN is initialized such that $B = C^\top$ and $A=A^\top$.
Then, if GD converges to a solution minimizing the loss in Equation~\eqref{eq:expected_loss}, this solution necessarily extrapolates.
\end{theorem}
\begin{proof}
The proof proceeds in two steps: we first show that GD preserves a few properties throughout training, and then establish that with these properties in place, any solution $(A,B,C)$ minimizing the loss must satisfy $CB=w^*$ and~$C A^j B$=0 for all $j \in \mathbb{N}$~---~conditions equivalent to extrapolation (see Lemma~\ref{lemma:extrapolation_conditions}).

Denote the training loss in~Equation~\eqref{eq:expected_loss} by $\mathcal{L} ( A , B , C )$.
A simple computation of derivatives (provided in  Appendix~\ref{app:derivatice_calc}) yields:
\begin{equation}\label{eq:grad_b}
    \frac{\partial \mathcal{L}}{\partial B}=\sum_{i=1}^{k-1}(A^\top)^iC^\top CA^iB+C^\top(CB-w^*),
\end{equation}
\begin{equation}\label{eq:grad_c}
\frac{\partial \mathcal{L}}{\partial C}=\sum_{i=1}^{k-1}CA^iBB^\top(A^\top)^i+(CB-w^*)B^\top ,
\end{equation}
and
\begin{equation}\label{eq:grad_a}
\frac{\partial \mathcal{L}}{\partial A}=\sum_{i=1}^{k-1}\sum_{r=0}^{i-1}(A^\top)^rC^\top CA^iBB^\top(A^\top)^{i-r-1}.    
\end{equation}
Denote by $A_t$, $B_t$ and $C_t$ the weights of the learned model at iteration~$t \in \mathbb{N} \cup \{ 0 \}$ of GD.
We will prove by induction that $B_t=C_t^\top $ and $A_t=A_t^\top$ for all~$t$. 
By assumption this holds for $t=0$. 
Suppose it is true for some $t$.
The GD updates for $B$ and $C$ are given by:



\begin{align}\label{eq:gd_update_b}
B_{t+1}=B_t-\eta\sum_{i=1}^{k-1} (A_t^\top)^i & C_t^\top C_tA_t^iB_t\\
& -\eta C_t^\top (C_tB_t-w^*),\nonumber
\end{align}


\begin{align}\label{eq:gd_update_c}
C_{t+1}=C_t-\eta\sum_{i=1}^{k-1} C_tA^i_t & B_tB_t^\top (A_t^\top )^i+\\
& -\eta(C_tB_t-w^*)B_t^\top.\nonumber
\end{align}

Taking transpose of the right-hand side of Equation~\eqref{eq:gd_update_c}, while noting that by our inductive hypothesis $B_t = C_t^\top$ and $A_t = A_t^\top$, we obtain equality to the right-hand side of Equation~\eqref{eq:gd_update_b}, where we used the fact that $CB-w^*$ is a scalar.
This implies $B_{t + 1} = C_{t + 1}^\top$.

As for~$A$, its GD update is:
\begin{equation}\label{eq:gd_update_a}
    A_{t+1}=A_t-\eta \sum_{i=1}^{k-1}\sum_{r=0}^{i-1}\gamma_{t,i,r}
\end{equation}
where 
\begin{equation*}
    \gamma_{t,i,r}=(A_t^\top)^rC_t^\top C_tA_t^iB_tB_t^\top(A_t^\top)^{i-r-1}
\end{equation*}
For any~$i$ between~$1$ and $k - 1$, the internal summation (over~$r$) is symmetric. 
To see this, let $W=C_t^\top C_tA_t^iB_tB_t^\top $, and note that~$W$ is a symmetric matrix (since $B_t=C_t^\top $ and $A_t$ is symmetric by our inductive hypothesis).
Now, for every term $(A_t^\top )^rW(A_t^\top )^{i-1-r}$, a corresponding term $(A_t^\top )^{i-1-r}W(A_t^\top )^{r}$ also appears in the summation, and these two terms together form a symmetric matrix (an exception is the case $i-1-r=r$, which corresponds to itself but is already symmetric). 
We conclude that the GD update in Equation~\eqref{eq:gd_update_a} can be written as a sum of symmetric matrices, and is therefore itself symmetric.
That is, $A_{t+1} = A_{t+1}^\top$, and our inductive hypothesis is proven.

Moving on to the second part of the proof, let $(A,B,C)$ be a minimizer of $\mathcal{L} ( A , B , C )$ satisfying $B = C^\top$ and $A = A^\top$.
By the structure of $\mathcal{L} ( A , B , C )$ (Equation~\eqref{eq:expected_loss}), it holds that $CB=w^*$ and $CA^jB=0$ for any $j \in \{1,\ldots,k-1\}$.
We will show that $C A^j B=0$ for all $j\in \mathbb{N}$.
Recalling that $k > 2$, we have in particular:
\begin{equation}\label{eq:psd_form_vanishes}
    CA^2B=0.
\end{equation}
$A$~is symmetric and therefore orthogonally diagonalizable, meaning there exists an orthogonal matrix $V \in \mathbb{R}^{d \times d}$ and a diagonal matrix $D \in \mathbb{R}^{d \times d}$ such that $A = V D V^\top$.
We can thus write $A^2 = V D^2 V^\top $, and since $B = C^\top $, Equation~\eqref{eq:psd_form_vanishes} implies:
\begin{equation*}
    CA^2B=B^\top A^2B=B^\top VD^2V^\top B=0.
\end{equation*}
Denoting $\mathbf{u}= (u_1 , \ldots , u_d )^\top = V^\top B$, we may write the above as
\begin{equation}
    \mathbf{u}^\top D^2\mathbf{u}=\sum_{i=1}^d u_i^2\lambda_i^2=0 ~.
    \label{eq:u_lambda_eq}
\end{equation}
Since this is a sum of non-negative elements that sum to zero, each of them must be zero, namely:
\begin{equation}
    u_i \lambda_i = 0,\quad i=1,\ldots, d ~.
    \label{eq:comp_slack}
\end{equation}
We refer to this as a {\em complementary slackness} condition, since it implies that either $u_i$ or $\lambda_i$ should be zero for any $i$. 
%
For arbitrary $j \in \mathbb{N}$: 
\begin{equation*}
    CA^jB=B^\top VD^jV^\top B=\mathbf{u}^\top D^j\mathbf{u}=\sum_{i=1}^d u_i^2\lambda_i^j = 0,
\end{equation*}
where the equality to zero follows from Equation \eqref{eq:comp_slack}. This is precisely the condition we set out to prove. 
\end{proof}

\begin{remark}
The assumption $w^*>0$ can easily be converted to $w^*<0$, by modifying the conditions on initialization to include $B=-C^\top$ instead of $B=C^\top$.
\end{remark}



Key to the proof of Theorem~\ref{thm:short_sequence_bias} is the fact that symmetry is invariant under GD, i.e.~if the model weights are symmetric at initialization, they remain that way throughout.
A natural question which arises is whether non-extrapolating solutions such as those described in Section~\ref{sec:bad_solutions} can be expressed with a symmetric weight configuration. The following lemma shows that there exist symmetric weight configurations with arbitrarily small loss values that do not extrapolate.

\begin{lemma}\label{lemma:bad_symmetric_solution}
Assume $d\ge k\ge2$ and $w^*>0$. For any $\epsilon>0$, there exists a weight configuration $(A,B,C)$ where $B=C^\top$ and $A=A^\top$, such that the loss in Equation~\eqref{eq:expected_loss} is smaller than~$\epsilon$ yet the model does not extrapolate.
\end{lemma}
\begin{proof}
We present a proof for $k=3$ and $d=4$.
Extension to arbitrary values of $k$ and~$d$ (satisfying $d\ge k\ge2$) is straightforward.

Let $C=B^\top=(\sqrt{w^*},0,0,\sqrt{\delta})$, $A=diag(0,0,0,2)$, where $\delta > 0$. The loss in Equation~\eqref{eq:expected_loss} is then: 
\begin{equation*}
    (CA^2B)^2+(CAB)^2+(CB-w^*)^2=16\delta^2+4\delta^2+\delta^2
    \text{\,.}
\end{equation*}
This is smaller than $\epsilon$ if $\delta<\sqrt{\frac{\epsilon}{21}}$. On the other hand, 
when tested on sequences of length $\tilde{k}$, the loss will be $\sum_{i=0}^{\tilde{k}} 2^{2i}\delta^2$, which diverges with $\tilde{k}$. $(A,B,C)$ therefore do not extrapolate.
\end{proof}



\section{EXPERIMENTS}\label{sec:experiments}
In this section we support our theoretical findings with several synthetic experiments demonstrating an implicit bias of gradient-based optimization towards extrapolating solutions.
The experiments cover not only linear RNNs (the subject of our theory), but also non-linear recurrent models including Long-Short Term Memory (LSTM) \cite{hochreiter1997long} and Gated Recurrent Units (GRU) \cite{cho2014learning}. Unless stated otherwise, in all experiments we use (non-symmetric) Xavier initialization \cite{glorot2010understanding}. For optimization we use Adam \cite{kingma2017adam} with learning rate $10^{-3}$ and default momentum parameters of Keras implementation ($\beta_1=0.9$ and $\beta_2=0.999$). Data is generated from a standard normal distribution with identity covariance, we experiment with many training sequences to ensure good generalization for the training sequence length, thereby decoupling it from the question of extrapolation.
Models had width (hidden state dimension) $d = 30$, and were trained with sequences of length $k = 5$. All experiments were run on a simple Colab client.

Section \ref{sec:bad_solutions} showed that when $d\ge k$, there exist solutions that perfectly fit training data yet fail to extrapolate to longer sequences. In order to demonstrate this empirically we ``adversarially'' learn models using the length~$k$ training sequences, while including erroneous (non-extrapolating) values for later time steps. This strategy is similar to that of learning with noisy labels, see~\cite{ZhangBHRV17}. We refer to models learned in this fashion as ``adversarial data'' in our figures.
Figure~\ref{fig:memoryless_models} reports an extrapolation experiment in the case of a memoryless teacher (in accordance with our theory). Three architectures are learned~---~linear RNN, LSTM and GRU~---~and are evaluated on time steps up to~$15$ (recall that training sequences are of length~$5$). For each architecture, we compare the result of with vs. without adversarial learning as described above.
The results confirm that all architectures extrapolate despite existence of non-extrapolating solutions.

\begin{figure*}[htp]
\centering
\includegraphics[width=0.7\textwidth, height=4.6cm]{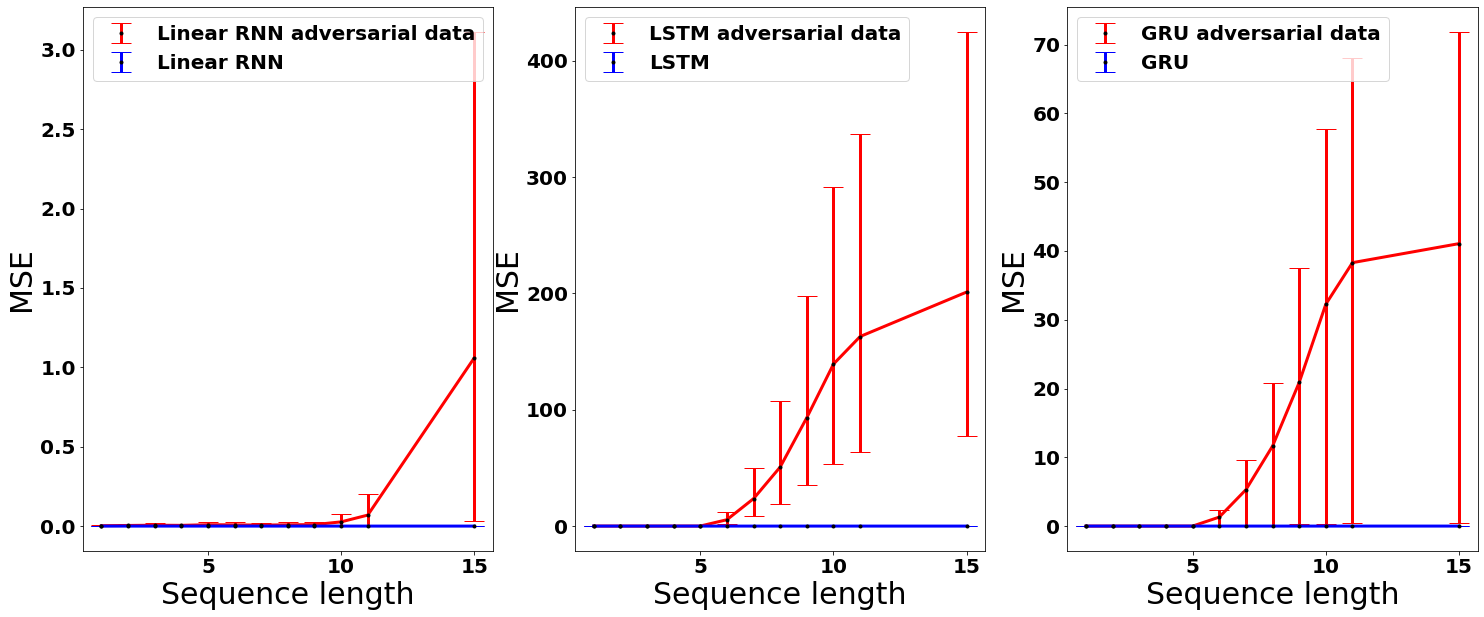}
\caption{Mean squared error over sequences of different lengths, after learning from length~$5$ training sequences generated from a memoryless teacher.
As can be seen, despite the fact that using adversarial data it is possible to fit training sequences with non-extrapolating solutions, gradient-based optimization leads to extrapolation.
}
\label{fig:memoryless_models}
\end{figure*}

Although our theory applies to a memoryless teacher,  the question of temporal extrapolation is relevant for teachers with arbitrarily long memory. Namely, if training data is generated from a teacher with state space of dimension~$d^*$ on $k$ time units, will a model learned via gradient-based optimization extrapolate? In light of our findings thus far, one may hope that extrapolation also occurs for non-zero $d^*$. Figure~\ref{fig:models_with_memory} confirms that this is indeed the case, via an experiment analogous to that of Figure~\ref{fig:memoryless_models} but with $d^*=3$.

\begin{figure*}[ht!]
\centering
\includegraphics[width=0.7\textwidth, height=4.6cm]{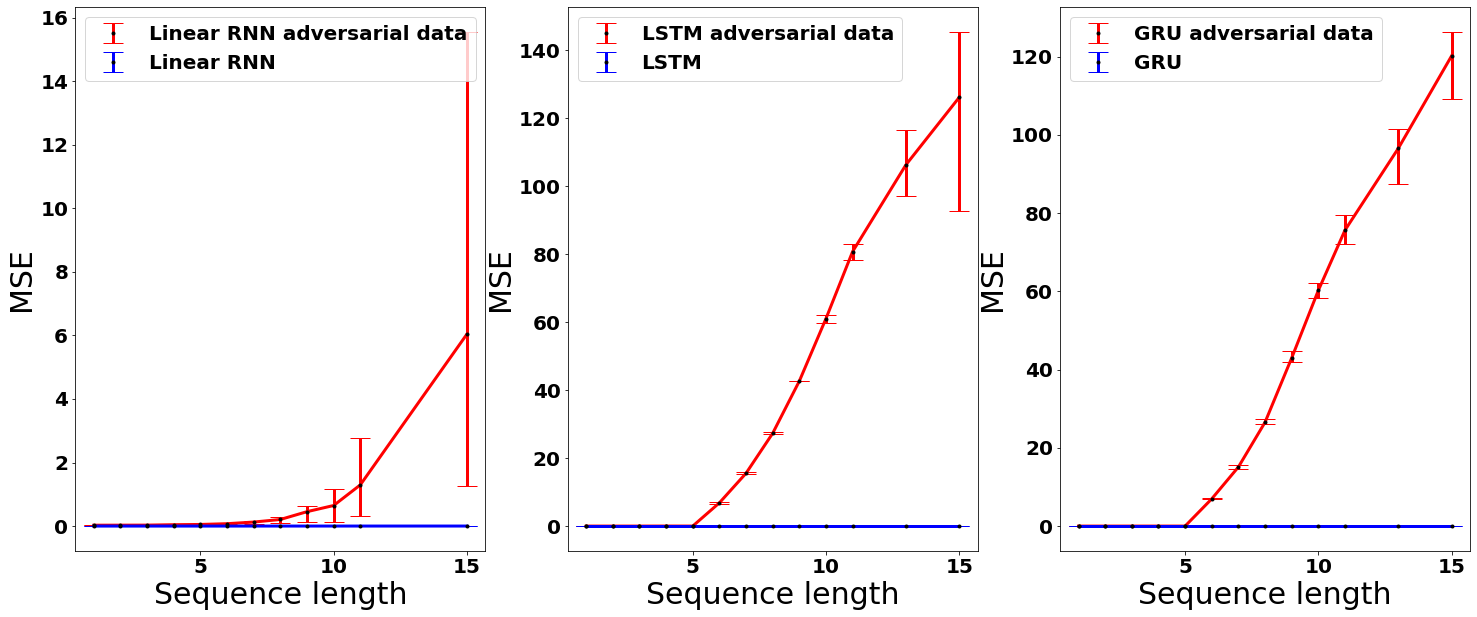}
\caption{
Mean squared error over sequences of different lengths, after learning from length~$5$ training sequences generated from a teacher with memory (state space of dimension~$3$).
As can be seen, despite the fact that using adversarial data it is possible to fit training sequences with non-extrapolating solutions, gradient-based optimization leads to extrapolation.
}
\label{fig:models_with_memory}
\end{figure*}

Next, we empirically demonstrate the complementary slackness phenomenon discussed in the proof of Theorem~\ref{thm:short_sequence_bias}. The proof suggests that any non-zero eigenvalue of~$A$ must align with zero entries of the projections of $B,C$ onto the orthonormal eigen-basis of $A$. Figure~\ref{fig:alignment_of_eigenvectors} demonstrates that this is indeed the case, for a linear RNN learned via gradient-based optimization  from a memoryless teacher.

\begin{figure}[hbp]
\centering
\includegraphics[width=0.4\textwidth]{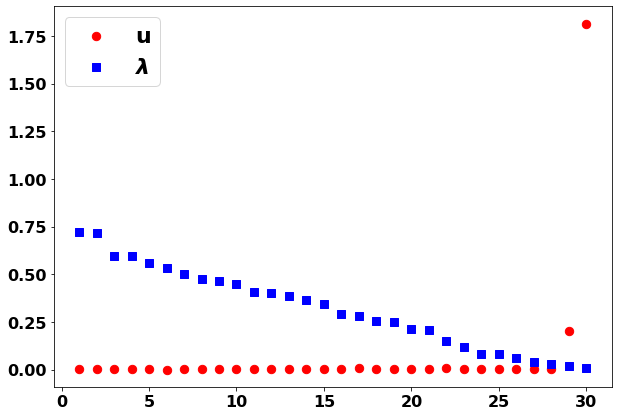}
\caption{Empirical demonstration of the complementary slackness phenomenon from the proof of Theorem~\ref{thm:short_sequence_bias}.
In blue are the absolute values of the eigenvalues of the state transition matrix~$A$. The vector $\mathbf{u}$ is the projection of the input and output weights $B$ and~$C$ (respectively) onto the orthonormal basis of~$A$, i.e. $\mathbf{u}:=V^\top B$ in the notations of the proof of Theorem~\ref{thm:short_sequence_bias}. The proof shows that the implicit bias of GD ensures that $\lambda_i u_i=0$ for all $i$. The results above validate this phenomenon.
}
\label{fig:alignment_of_eigenvectors}
\end{figure}

\subsection{Weight Dynamics}\label{subsec:weight_dynamics}
Section~\ref{sec:implicit_bias_of_short_sequences} showed that when initializing a linear RNN symmetrically ($A^\top=A$ and $B=C^\top$), GD is guaranteed to preserve symmetry, and consequently converge to an extrapolating solution. In this experiment we optimize the population loss directly (e.g. Equation~\ref{eq:expected_loss}) using GD as to observe the weight dynamics leading to extrapolation.
Figure~\ref{fig:dynamics_of_b_and_c} below suggests that GD exhibits a tendency towards symmetry even when initialization is non-symmetric. 
It displays the evolution of weights during optimization when $A$ is initialized as $A_0=\alpha I$ with random $\alpha\in[0,1]$, and $B,C$ are initialized independently from a random normal distribution. 
As can be seen, weights converge to an approximately symmetric solution, in the sense that the norms of $A-A^\top$ and $B-C^\top$ are much smaller than those of $A, B, C$.
We hypothesize that this is due to conservation laws of the GD dynamics, akin to those studied in \cite{saxe2013exact, kunin2020neural}. Their derivation is left for future work. 
\begin{figure}[ht!]
\centering
\includegraphics[width=0.45\textwidth]{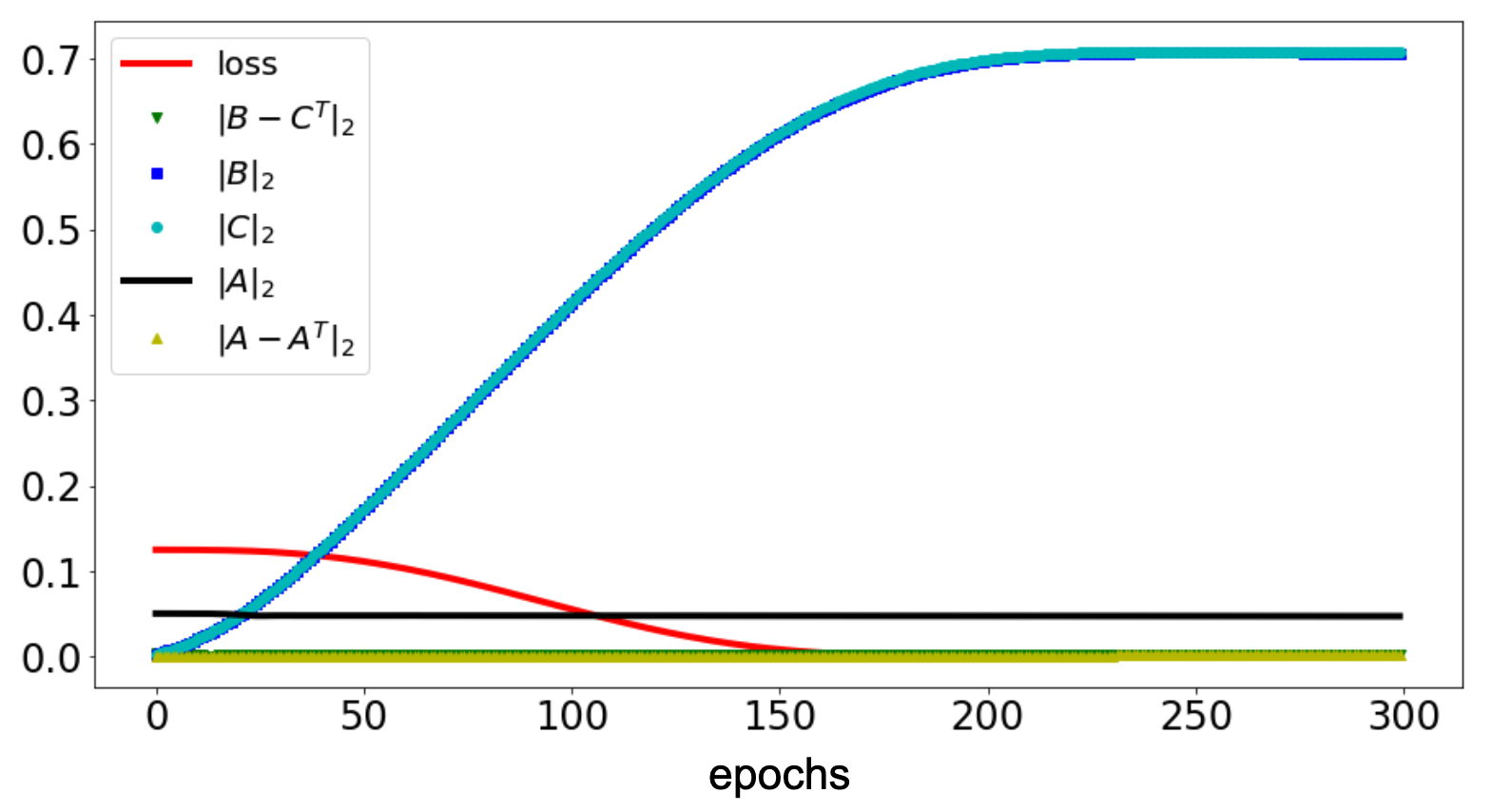}
\caption{Dynamics of $A,B,C$~---~weights of linear RNN~---~under GD, when $B,C$ are initialized independently (non-symmetrically) from a zero-centered Gaussian distribution with variance $\sigma^2=10^{-5}$, and $A$ is initialized as scaled identity. The figure shows the loss, the (Euclidean) norms of the weights, and as a measure of symmetry, the norms of $B-C^\top$ and $A-A^\top$. As can be seen, the norms of $B,C$ grow while the weights remain approximately symmetric.} 
\label{fig:dynamics_of_b_and_c}
\end{figure}



\section{CONCLUSIONS}\label{sec:conclusions}

In this paper we studied the implicit bias of gradient descent (GD) in the context of temporal extrapolation. Focusing on linear recurrent neural networks (RNNs), also known as linear dynamical systems, we analyzed the setting of unlimited training data generated from a memoryless teacher network, and proved that when the width of the learned model is greater than the length of training sequences, there exist solutions that do not extrapolate, yet GD will converge to solutions that do. We showed that this is a result of a complementary slackness phenomenon between the eigenvalues of the state transition matrix~$A$ and the input and output weights $B$ and~$C$ respectively. 

Our theory imposes certain assumptions on initialization, and is limited to a memoryless teacher.
However, we demonstrate empirically that gradient-based optimization exhibits an implicit bias towards extrapolation even without these restrictions, in particular using standard initialization schemes and teachers with memory.
Moreover, our experiments confirm that the phenomenon extends to non-linear RNNs including GRU and LSTM.
We believe elements of our theory may prove useful in analyzing non-linear RNNs, and view this pursuit as an direction for future work.

%

Our work extends the rich body of literature studying implicit biases of GD in neural networks, by treating the important class of temporal (recurrent) models, and in particular the question of temporal extrapolation. We believe our results may contribute to a better understanding of when extrapolation fails or succeeds, thereby facilitating learning algorithms that improve time series prediction.

\section{Acknowledgements}
This work was supported by a Google Research
Scholar Award, a Google Research Gift, the Yandex Initiative in Machine Learning, the Israel Science Foundation (grant 1780/21), the European Research
Council (ERC) under the European Unions Horizon 2020
research and innovation programme (grant ERC HOLI
819080), Len Blavatnik and the Blavatnik Family Foundation, and Amnon
and Anat Shashua.

\bibliography{sample_paper}


\clearpage
\appendix

\thispagestyle{empty}

\onecolumn \makesupplementtitle
\section{Proof of Lemma 3.1 (Main Paper): Population Risk for SISO}\label{app:missing_proofs}
\begin{lemma}{3.1}[Main Text]\label{lemma:expected_loss_apndx}
Assume $\mathbf{x}\sim \mathcal{D}$ such that $\mathbb{E}[\mathbf{x}]=0$ and $\mathbb{E}[\mathbf{x}\mathbf{x}^T]=I_k$, where $I_k \in \mathbb{R}^{k,k}$ is the identity matrix. Then, given a memoryless teacher RNN, the loss for the student RNN is given by:
\begin{equation}\label{eq:expected_loss_apdx}
    \mathbb{E}_{\mathbf{x},y}\left[\frac{1}{2}\left(RNN(\mathbf{x})-y\right)^2\right]=\frac{1}{2}\sum_{i=1}^{k-1}(CA^{k-i}B)^2+\frac{1}{2}(CB-w^*)^2
\end{equation}
\end{lemma}
\begin{proof}
\begin{equation*}
    \mathbb{E}_{\mathbf{x},y}\left[\frac{1}{2}\left(RNN(\mathbf{x})-y\right)^2\right]=\frac{1}{2}\mathbb{E}\left[ \left( \sum_{i=1}^k  CA^{k-i}Bx_i -w^*x_k \right)^2 \right]
\end{equation*}
The above can be written as

\begin{equation}
    \frac{1}{2}\mathbb{E}\left[ \left(\sum_{i=1}^k\left( CA^{k-i}Bx_i \right)\right)^2 +(w^*x_k)^2-2\sum_{j=1}^k CA^{k-j}Bx_j x_kw^*\right]
\end{equation}
Because $\mathbf{x}$ has identity covariance ($\mathbb{E}[x_i^2]=1$) many terms cancel out and the above is equal to
\begin{equation*}
    \frac{1}{2}\left[\sum_{i=1}^k (CA^{k-i}B)^2+ \left ((w^*)^2-2CBw^* \right )\right]
\end{equation*}
Removing $i=k$ from the summation, we have,
\begin{equation*}
    \frac{1}{2}\sum_{i=0}^{k-1}(CA^{k-i}B)^2+\frac{1}{2}\left ((CB)^2-2CBw^*+(w^*)^2\right )
\end{equation*}
The above can be written as
\begin{equation*}
    \frac{1}{2}\sum_{i=0}^{k-1}(CA^{k-i}B-0)^2+\frac{1}{2}(CB-w^*)^2
\end{equation*}
which concludes the proof.
\end{proof}
\vfill

\section{Details for Proof of Theorem 6.1}\label{app:derivatice_calc}
Here we the gradient computation for the proof of Theorem 6.1 in the main text.

Consider the expected loss in Lemma~\ref{lemma:expected_loss_apndx}. For $\frac{1}{2}(CB-w^*)^2$ the derivative w.r.t $B$ is
\begin{equation}\label{eq:grad_b_rhs}
    \frac{\partial \frac{1}{2}(CB-w^*)^2}{\partial B}=C^T(CB-w^*)
\end{equation}

For $j\ge 1$, the derivative of $\frac{1}{2}(CA^jB)^2$ w.r.t to $B$ is given by
\begin{equation}\label{eq:grad_b_lhs}
    \frac{\partial \frac{1}{2}(CA^jB)^2}{\partial B}=(A^j)^TC^TCA^jB=(A^T)^jC^TCA^jB
\end{equation}

Putting together Equations \eqref{eq:grad_b_rhs} and \eqref{eq:grad_b_lhs}, the derivative of (4) in the main text w.r.t. $B$ is given by
\begin{equation*}
    \frac{\partial \mathcal{L}}{\partial B}=\sum_{i=1}^{k-1}(A^T)^iC^TCA^iB+C^T(CB-w^*)
\end{equation*}

A similar derivation w.r.t. $C$ yields:
\begin{equation*}
    \frac{\partial \mathcal{L}}{\partial C}=\sum_{i=1}^{k-1}CA^iBB^T(A^T)^i+(CB-w^*)B^T
\end{equation*}

For the gradient w.r.t. $A$, $\forall i\ge 1$, the derivative of $\frac{1}{2}(CA^iB)^2$ is based on Equation (91) from \cite{petersen2012matrix}.
\begin{equation*}
    \frac{\partial \frac{1}{2}(CA^iB)^2}{\partial A}=\sum_{r=0}^{i-1}(A^r)^TC^TCA^iBB^T(A^{i-1-r})^T
\end{equation*}

Using $(A^j)^T=(A^T)^j$ and summing over $i=1,\dots,k-1$ results in:
\begin{equation}
    \frac{\partial \mathcal{L}}{\partial A}=\sum_{i=1}^{k-1}\sum_{r=0}^{i-1}(A^T)^rC^TCA^iBB^T(A^T)^{i-r-1}
\end{equation}
\section{Multiple Input Multiple Output}\label{sec:appendix_mimo}
In this section we discuss the extension of our results to the case of \textit{Multiple Input Multiple Output} (MIMO) systems. In what follows we denote the input dimension by $n$, and the output dimension, $m$.

Consider an RNN with hidden width $d$, input sequence $\{ X_t \}_{t = 1}^\infty \subset \mathbb{R}^n$ representing a sequence of $n$-dimensional inputs, denote the $i^{th}$ column of $X$ by $X_i$. The model produces outputs $\{ \hat{\mathbf{y}}_t \}_{t = 1}^\infty \subset \mathbb{R}^m$ through the following update equations:
\begin{equation}\label{eq:output_of_rnn_mimo}
    \hat{\textbf{y}}_t=C s_t,\qquad s_{t+1}=As_t+B X_{t+1},
\end{equation}
where $A\in\mathbb{R}^{d\times d}$, $B\in\mathbb{R}^{d\times n}$ and $C\in\mathbb{R}^{m\times d}$ are the learned parameters, and $\{ s_t \}_{t = 1}^\infty \subset \mathbb{R}^d$ are the resulting hidden states, where by assumption $s_0=0$. Given an input sequence $X\in\mathbb{R}^{n\times k}$, a memoryless MIMO teacher corresponds to $W^*\in\mathbb{R}^{m\times n}$, such that $\mathbf{y}=W^*X_k$.

In the main paper we develop an expression for the population loss for the case of SISO. We provide here a  MIMO version of the lemma.

\begin{lemma}\label{lemma:expected_loss_mimo}
Assume $X\in\mathbb{R}^{n\times k}$, $X\sim\mathcal{D}$ such that $\mathbb{E}_{\mathcal{D}}[XX^\top]=I_{n}$ and $\mathbb{E}_{\mathcal{D}}[X]=0$. Then, given a memoryless teacher RNN, the loss for the student RNN satisfies:
\begin{equation}\label{eq:mimo_expected_loss}
    \mathbb{E}_{X,\mathbf{y}}\left[\frac{1}{2}\left\|RNN(X)-\mathbf{y}\right\|_F^2\right]=\frac{1}{2}\sum_{i=1}^{k-1}\left\|CA^{k-i}B\right\|_F^2+\frac{1}{2}\left\|CB-W^*\right\|_F^2
\end{equation}
\end{lemma}

\begin{proof}
The proof is given in \ref{appendix:mimo_proofs}.
\end{proof}

In the main paper we show that when the sequence length is greater than the hidden dimension, ($k>d$), extrapolation is guaranteed by showing that $\forall j\in \mathbb{N}$ it holds that $CA^jB=0$. The analysis in the main paper is not dependent on the dimensions of $B$ and $C$ and therefore applies to the MIMO setting as-is.

Following the analysis of extrapolation when learning with long sequences, we show that when $k<d$, there exists solutions that attain zero loss but do not extrapolate w.r.t. a memoryless teacher. The proof uses the following parameters,

\[
 A=\begin{pmatrix}  0 & 0 & \dots & 0 & 0 & 1 \\
                    1 & 0 & \dots & 0 & 0 & 0\\
                    0 & 1 & \dots & 0 & 0 & 0\\
                    & & \ddots & & & \\
                    0 & 0 & \dots & 1 & 0 & 0 \\
                    0 & 0 & \dots & 0 & 1 & 0 
    \end{pmatrix} 
    \in \mathbb{R}^{d , d} ~,
\quad
B=\begin{pmatrix}1\\ 0\\ \vdots\\ 0 \end{pmatrix} 
    \in \mathbb{R}^{d , 1} ~,
\quad
C=(w^*,0,\dots,0) \in \mathbb{R}^{1 , d}.
\]
In order to apply for MIMO, the parameters $B$ and $C$ need to be padded with zeros to form,
\[
B=\begin{pmatrix}1 & 0 & \dots & 0\\ 0 & 0 & \dots & 0\\ \vdots & \vdots & \ddots & \vdots\\ 0 & 0 & \dots & 0 \end{pmatrix} 
    \in \mathbb{R}^{d , n} ~,
\quad
C=\begin{pmatrix} w^* & 0 & \dots & 0\\ 0 & 0 & \dots & 0\\ \vdots & \vdots & \ddots & \vdots\\ 0 & 0 & \dots & 0 \end{pmatrix} \in \mathbb{R}^{m , d}.
\]
and the same arguments apply.

In the main paper we show GD has implicit bias towards memoryless solutions under standard initialization schemes. Here we show that this result extends to the MIMO case, under the additional conditions that $m=n$ and $W^*$ is symmetric and nonnegative.


The SISO proof follows two steps. The first shows that at any time step of GD, $B_t=C_t^T$ and $A_t$ is symmetric. For the first part of the proof to apply for the MIMO setting, the dimensions of $B$ and $C$ must allow $B_0=C_0^T$ which implies $n=m$. 

The gradient updates in the MIMO case are similar to those of SISO (see Appendix~\ref{app:missing_proofs}) and are given by:\begin{equation}\label{eq:mimo_grad_b}
    \frac{\partial \mathcal{L}}{\partial B}=\sum_{i=1}^{k-1}(A^T)^iC^TCA^iB+C^T(CB-W^*),
\end{equation}
\begin{equation}\label{eq:mimo_grad_c}
\frac{\partial \mathcal{L}}{\partial C}=\sum_{i=1}^{k-1}CA^iBB^T(A^T)^i+(CB-W^*)B^T ,   
\end{equation}
\begin{equation}\label{eq:mimo_grad_a}
\frac{\partial \mathcal{L}}{\partial A}=\sum_{i=1}^{k-1}\sum_{r=0}^{i-1}(A^T)^rC^TCA^iBB^T(A^T)^{i-r-1}.    
\end{equation}
The same inductive argument from the main text applies here with the distinction that in order for the RHS of \eqref{eq:mimo_grad_b} and \eqref{eq:mimo_grad_c} to satisfy
\begin{equation*}
    C_t^\top(C_tB_t-W^*)=\left[(C_tB_t-W^*)B_t^\top \right]^\top
\end{equation*}
the matrix $W^*$ must be symmetric.

For the second part of the proof, we use the fact that $B = C^\top$ and $A = A^\top$ to show that at convergence $C A^j B=0$ for all $j\in \mathbb{N}$.
%
Recalling that $k > 2$, consider the optimized loss:\footnote{The leftmost term is zero by definition if $k=3$.}
\begin{equation}\label{eq:loss_for_k_ge_3_mimo}
    \mathcal{L} ( A , B , C ) = \sum_{i=3}^{k-1}\|CA^iB\|_F^2+\|CA^2B\|_F^2+\|CAB\|_F^2+\|CB-W^*\|_F^2
\end{equation}
Any solution minimizing (i.e., bringing to zero) the above must satisfy:
\begin{equation}\label{eq:psd_form_vanishes_mimo}
    CA^2B=0\in\mathbb{R}^{n\times n}. 
\end{equation}
By the assumption of the theorem we have that GD converges to a minimizing solution and therefore satisfies Equation \eqref{eq:psd_form_vanishes}.
    Also, by the first part of the proof, we know that $A$~is symmetric and therefore orthogonally diagonalizable, meaning there exist an orthogonal matrix $V \in \mathbb{R}^{d,d}$ and a diagonal matrix $D \in \mathbb{R}^{d , d}$ such that $A = V D V^\top$.
    We can thus write $A^2 = V D^2 V^\top$, and since $B = C^\top$ (by the first part of the proof),
    
    Equation~\eqref{eq:psd_form_vanishes} implies:
    
    \begin{equation*}
        CA^2B=B^\top A^2B=B^\top VD^2V^\top B=0\in\mathbb{R}^{n\times n}.
    \end{equation*}
    Denote $U=V^\top B$, the above can be written as
    \begin{equation*}
       B^\top VD^2V^\top B = U^\top D^2U  = 0\in\mathbb{R}^{n\times n}. 
    \end{equation*}
    The above matrix is element-wise zero, in particular its diagonal elements should be zero, implying for all $i$:
 \begin{equation}\label{eq:mimo_comp_slack}
 \left[U^\top D^2U\right]_{ii}=\sum_{s=1}^d U^2_{si}\lambda_s^2 = 0
 \end{equation}
Since Equation~\eqref{eq:mimo_comp_slack} is a sum of non-negative elements that sum to zero, each of them should be zero. Furthermore, for any $s$, it must hold that $U_{si}^2\lambda_s^2=0$ and therefore we have the complementary slackness result:
\begin{equation}\label{eq:comp_slack_element}
    U_{si}\lambda_s = 0 \ \ \forall i,s
\end{equation}
    
The fact that the model extrapolates follows directly from the observation above. Consider any $p \in \mathbb{N}$. Then the corresponding element in the impulse response is given by:
\begin{equation*}
CA^pB=B^\top VD^pV^\top B=U^\top D^pU
\end{equation*}
which can be written as
\begin{equation*}
    \left[ U^\top D^pU \right]_{ij} = \sum_{s=1}^d U_{si}U_{sj}\lambda_s^p
\end{equation*}
From Equation \eqref{eq:comp_slack_element} we conclude that the above is zero and thus $CA^pB=0$ (i.e., this part of the matrix impulse response is zero).
This is precisely the condition for perfect extrapolation (see main text and recall that $CB-W^*=0$ because of optimality of GD) and thus the result follows.

\subsection{Population Loss for MIMO}\label{appendix:mimo_proofs}

\begin{proof}[Proof for Lemma~\ref{lemma:expected_loss_mimo}]
\begin{equation*}
    \mathbb{E}_{X,\mathbf{y}}\left[\frac{1}{2}\left\|RNN(X)-\mathbf{y}\right\|_F^2\right]=\frac{1}{2}\mathbb{E}\left[ \left\| \sum_{i=1}^k  CA^{k-i}BX_i -W^*X_k \right\|_F^2 \right]
\end{equation*}

For two general matrices $Q,R$, the loss in terms of the trace operator is given by
\begin{align}\label{eq:mimo_trace}
    \|Q-R\|_F^2&=tr(\left(Q-R\right)^T\left(Q-R\right))\nonumber\\
    &=tr\left(Q^TQ-Q^TR-R^TQ+R^TR\right)\nonumber\\
    &=tr(Q^TQ)-tr(Q^TR)-tr(R^TQ)+tr(R^TR)\nonumber\\
    &=tr(Q^TQ)-2tr(Q^TR)+tr(R^TR)
\end{align}
where the transitions rely on the properties of the trace operator. We can now handle each term separately, denote $W_i=CA^{k-i}B$, assigning $Q=\sum_{i=1}^{k}W_iX_i$, the LHS term, $tr(Q^TQ)$, amounts to
\begin{align*}
    tr\left( \left(\sum_{i=1}^kX_i^TW_i^T \right) \left(\sum_{j=1}^kW_jX_j \right) \right)&=tr\left(\sum_{i=1}^{k}\sum_{j=1}^{k}X_i^TW_i^TW_jX_j\right)\\
    &=\sum_{i=1}^{k}\sum_{j=1}^{k}tr\left(X_i^TW_i^TW_jX_j\right)\\
    &=\sum_{i=1}^{k}\sum_{j=1}^{k}tr\left(W_i^TW_jX_jX_i^T\right)
\end{align*}
Taking the expectation of IID samples $X_i,X_j$,
\begin{equation*}
    \mathbb{E}\left[ tr\left(W_i^TW_jX_jX_i^T\right) \right]=tr\left(W_i^TW_j\mathbb{E}\left[X_jX_i^T\right]\right)=\begin{cases}0 & i\neq j\\ tr(W_i^TW_j) & i=j\end{cases}
\end{equation*}

putting together, the LHS term amounts to
\begin{equation}\label{eq:mimo_loss_term1}
    \mathbb{E}\left[tr(Q^TQ)\right]=\sum_{i=1}^{k}tr(W_i^TW_i)=\sum_{i=1}^{k}\|W_i\|_F^2=\sum_{i=1}^{k}\left\|CA^{k-i}B\right\|_F^2
\end{equation}

For the second term, $tr(Q^TR)$, we have
\begin{equation*}
    tr\left( \sum_{i=1}^{k}X_i^TW_i^TW^*X_k \right)=\sum_{i=1}^{k}tr\left( X_i^TW_i^TW^*X_k \right)=\sum_{i=1}^{k}tr\left( W_i^TW^*X_kX_i^T \right)
\end{equation*}
Taking the expectation, for every $i\neq k$, $\mathbb{E}\left[ X_kX_i^T\right]=0$, and for $i=k$, $\mathbb{E}\left[ X_kX_k^T \right]=I_n$. Therefore the middle term amounts to
\begin{equation}\label{eq:mimo_loss_term2}
    \mathbb{E}\left[ tr(Q^TR) \right]=tr\left(W_k^TW^*\right)
\end{equation}
Finally, the RHS is given by
\begin{equation}\label{eq:mimo_loss_term3}
    \mathbb{E}\left[ tr(R^TR) \right]=tr\left( (W^*)^TW^*\right)
\end{equation}
where we again use the linearity and cyclic properties of the trace operator as well as $\mathbb{E}\left[X_kX_k^T\right]=I_n$.

Putting the computed terms, \eqref{eq:mimo_loss_term1} \eqref{eq:mimo_loss_term2} \eqref{eq:mimo_loss_term3}, back into Equation~\eqref{eq:mimo_trace}, we have
\begin{equation*}
\mathbb{E}\left[ \left\| \sum_{i=1}^k  CA^{k-i}BX_i -W^*X_k \right\|_F^2 \right]=\sum_{i=1}^{k}\left\|CA^{k-i}B\right\|_F^2 -2tr\left(W_k^TW^*\right)+tr\left( (W^*)^TW^*\right)
\end{equation*}
Note that $W_k=CA^{k-k}B=CB$, the above can be written as
\begin{equation}
    \sum_{i=1}^{k-1}\left\|CA^{k-i}B\right\|_F^2+\|CB\|_F^2 -2tr\left((CB)^TW^*\right)+\|W^*\|_F^2
\end{equation}
which can further be written as
\begin{equation}
    \sum_{i=1}^{k-1}\left\|CA^{k-i}B\right\|_F^2+\|CB-W^*\|_F^2
\end{equation}
to conclude the proof.
\end{proof}

\section{Additional Experiments}
In the paper we show that GD has an inductive bias towards memory-less models. Namely, if the training data can be fit with a memory-less model, gradient descent with symmetric initialization will extrapolate well. Here we ask the more general question: if data is generated by a low dimensional LinearRNN (i.e., with low dimensional $A$), will GD extrapolate well. Namely, we ask whether gradient descent with symmetric initialization has an inductive bias towards low-dimensional systems.

Clearly, if the training sequences are shorter than the dimension of the ground-truth $A$, we should not expect to extrapolate well (since the short sequence does not capture the full behavior of the true model).

In what follows, we use $d^*$ to denote the dimension of $A$ for the ground-truth system. We let $k$ denote the length of the training data. Based on our discussion above, we would expect the following two regimes:
\begin{itemize}
    \item Good extrapolation for  $k\ge d^{*}$, since in this case there are sufficient observations to identify a low dimensional model and the data can be fit by this model. Moreoever, if GD with the said initialization scheme is indeed biased towards low order models, it will converge to the model with dimension $d^*$.
    \item Bad extrapolation for $k<d^*$ since in this case the first $k$ time units are insufficient to uniquely identify the ground-truth model.
\end{itemize}

We explore the above question using three different models, LinearRNN, GRU and LSTM with standard Xavier initialization. For all experiments, we set $k=5$, $d=200$ and $d^*=1,2,4,6,8$.

The architecture of the teacher is a LinearRNN with varying $d^*$. For each trained model, we estimate the extrapolation MSE as the average error of the model on sequence lengths $6, 7, 8, 9, 10$ (i.e., lengths it was not trained on). Figure \ref{fig:model_learning_increasing_memory} shows extrapolation error as a function of $d^*$. It can be seen that results are in line with the two regimes mentioned above. Namely, up to some point (roughly $d^*=k$), the model extrapolates well, and beyond this point extrapolation deteriorates.

These results suggest that gradient descent with standard initialization is indeed biased towards models with smaller dimensionality $d$. Furthermore, this happens for both linear and non linear models.  

\begin{figure}[h!]
\centering
\includegraphics[width=0.95\textwidth]{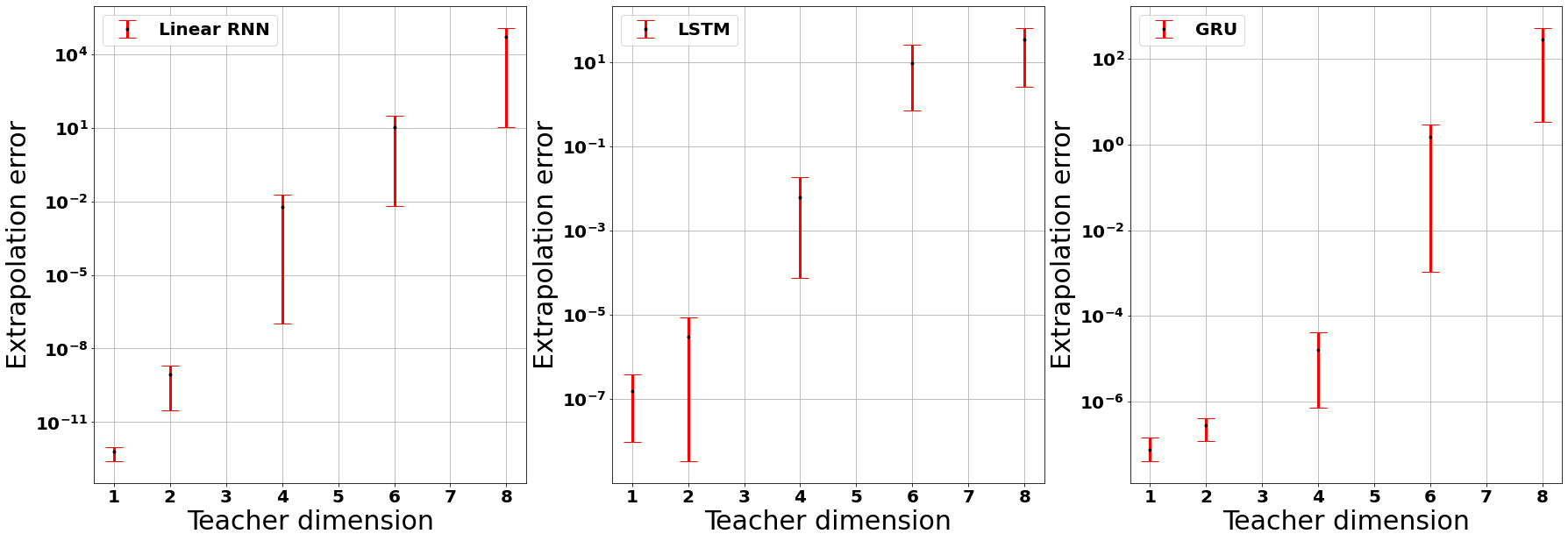}
\caption{Extrapolation error as a function of the teacher dimension. The figure shows that when $d^*<k$ there is good extrapolation indicating inductive bias towards low dimensional model. On the other hand for $d^*>k$ extrapolation fails, which is expected as the training examples are not long enough to reveal the teacher dynamics for sequences with length greater than $k$.}
\label{fig:model_learning_increasing_memory}
\end{figure}

\end{document}